%% file: neurips_2026.tex
\documentclass{article}

\usepackage[dblblindworkshop, final, nonatbib]{neurips_2026}
\workshoptitle{AI for Stochastic Dynamics}

\usepackage[utf8]{inputenc} 
\usepackage[T1]{fontenc}    
\usepackage{hyperref}       
\usepackage{url}            
\usepackage{booktabs}       
\usepackage{amsfonts}       
\usepackage{nicefrac}       
\usepackage{microtype}      
\usepackage{xcolor}         
\usepackage{amsmath}
\usepackage{booktabs}
\usepackage{longtable}
\usepackage{pdflscape}
\usepackage{array}
\usepackage{amsthm}
\usepackage{graphicx}
\usepackage{subcaption}
\usepackage{enumitem}

\newtheorem{corollary}{Corollary}
\newtheorem{proposition}{Proposition}
\newtheorem{definition}{Definition}

\usepackage[natbib=true]{biblatex}
\DeclareMathOperator*{\argmax}{arg\,max}
\setlist{leftmargin=1.5em}

\title{Answer-Distribution Trajectories: \\ A Stochastic-Dynamics View of LLM Reasoning}

\author{%
  Mar Gonzàlez I Català \\
  University of Cambridge\\
  \texttt{mg2211@cam.ac.uk} \\
  \And
  Haitz Sáez de Ocáriz Borde \\
  University of Cambridge \\
  \texttt{hs788@cam.ac.uk}
  \And
  Davide Murari \\
  University of Cambridge \\
  \texttt{dm2011@cam.ac.uk}
  \And
  Carola Bibiane-Schönlieb \\
  University of Cambridge \\
  \texttt{cbs31@cam.ac.uk} \\
  \And
  Pietro Liò \\
  University of Cambridge \\
  \texttt{pl219@cam.ac.uk}
  \And
  George D. Montañez \\
  Harvey Mudd College \\
  \texttt{gmontanez@g.hmc.edu} \\
}

\begin{document}

\maketitle

\begin{abstract}
Chain-of-thought reasoning provides a structured computation between a model’s input and final answer. Yet it is often evaluated through endpoint accuracy, which ignores the path taken to reach that answer. An emerging line of work addresses this limitation using entropy profiles, which track how uncertainty evolves over the reasoning process but do not reveal which competing hypotheses account for that uncertainty. We introduce answer-distribution trajectories, a stochastic-dynamics-inspired representation that tracks the model’s full predictive distribution over answers as reasoning unfolds. As a strictly finer representation than endpoint and entropy summaries, answer-distribution trajectories enable us to characterize a trace through a dynamical reasoning profile spanning exploration, revision, motion, and commitment, and to distinguish different dynamical mechanisms of reasoning success and failure. Across sixteen open-weight language models and four reasoning benchmarks, we show that traces with the same endpoint and similar entropy profiles can exhibit substantially different reasoning dynamics. We further find substantial variation in these dynamics both within and across models and tasks, with different objectives favoring different dynamical profiles. Additionally, we show that training and inference choices systematically reshape these profiles. Our results suggest that answer-distribution trajectories provide a rich framework for analysing and evaluating the dynamics of LLM reasoning.
\end{abstract}

\input{main}

\printbibliography

\newpage

\appendix
\input{AppendixA_experimental_setup}
\input{AppendixB_reproduction_figures_and_tables}
\input{AppendixC_computational_cost}
\input{AppendixD_proofs}
\input{AppendixE_licences}

\input{AppendixF_impact_statement}

\end{document}

%% file: main.tex
\section{Introduction}
\label{sec:introduction}

Chain-of-thought (CoT) reasoning first emerged as a powerful capability of large language models: when prompted to produce intermediate natural-language steps, sufficiently capable models could substantially improve performance on multi-step reasoning tasks \cite{kojima2022large, wei2022chain}. What began as a prompting technique has since become integrated into model training and post-training, leading to modern reasoning models that are explicitly optimized to generate extended intermediate computations before producing an answer \cite{Guo_2025, openai2026openaio1card}. In this setting, CoT is part of the model's inference-time computation, providing an observable trace of the model’s reasoning. Despite having access to this trace, reasoning models are still mostly evaluated by whether the final answer is correct, without accounting for the reasoning that produced it \cite{mao2024champ, prasad2023receval,xia2025evaluating,uesato2022solving}.

Recent work has therefore begun to study CoT traces and signals derived from them to better understand how a model arrives at its answer \cite{korbak2025chain, lightman2024let, wang2024multi}. In particular, a growing line of work examines entropy profiles, tracking how uncertainty over the next token or final answer evolves during CoT generation. These profiles have been used to guide exploration and early stopping \cite{sharma2025thinkjustenoughsequencelevel, zhang2025entropy}, identify critical decision points \cite{qian2025demystifyingreasoningdynamicsmutual, NEURIPS2025_a797c2d2}, and detect reasoning failures \cite{farquhar2024detecting, ton2025understandingchainofthoughtllmsinformation}. However, entropy is a limited signal: it captures how probability mass is distributed, but not which hypotheses carry that probability mass \cite{hullermeier2021aleatoric, wimmer2023quantifying}.

To address these limitations, we propose analyzing reasoning through \emph{answer-distribution trajectories}, which track a model’s predictive distribution over possible final answers as the CoT is generated (Figure~\ref{fig:answer-distribution-trajectory}). For a given trace, we estimate this distribution at selected prefixes by conditioning on the reasoning generated so far and sampling independent continuations. The resulting trajectory contains strictly more information than final-answer accuracy or entropy, as it preserves which hypotheses are supported, how that support evolves, and where it ultimately settles.

\begin{figure*}[t]
    \centering
    \includegraphics[width=\textwidth]{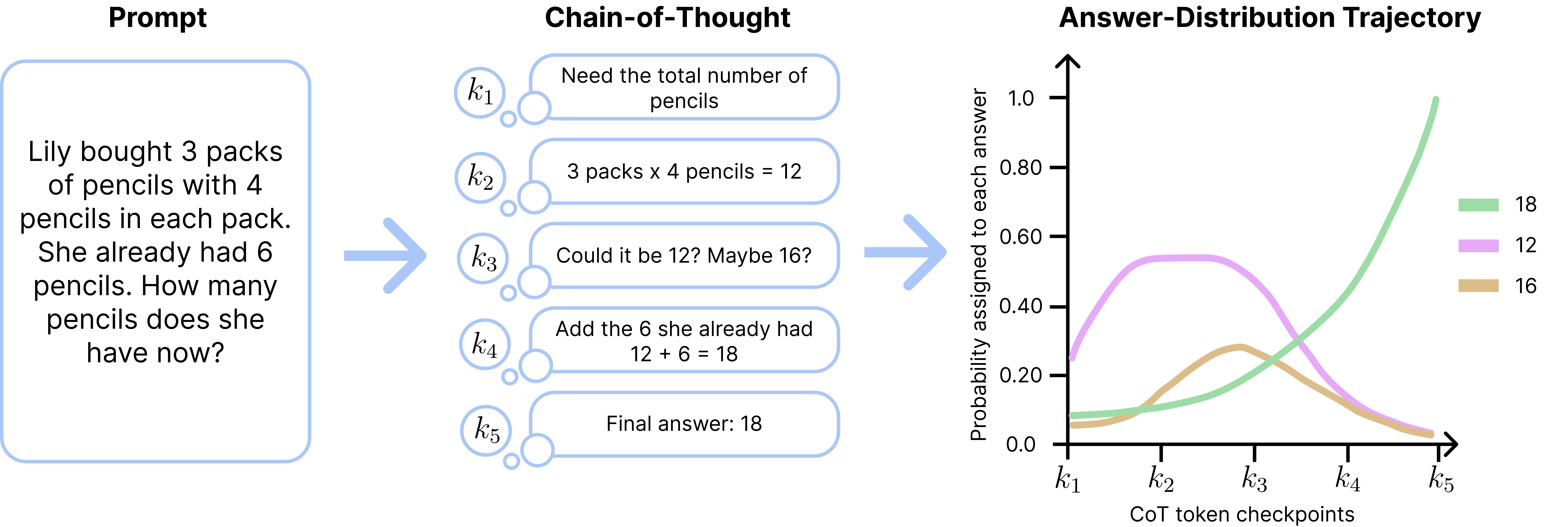}
    \caption{
        \textbf{Answer-distribution trajectories reveal how predictive mass moves among
        competing hypotheses during reasoning.}
        Here, the answer $12$ initially dominates before losing support as probability shifts toward the answer $18$. This switch in the dominant answer is not recoverable from the final answer or entropy alone.
    }
    \label{fig:answer-distribution-trajectory}

\end{figure*}

Answer-distribution trajectories allow us to define trajectory-level metrics that capture different aspects of the reasoning dynamics. We group these metrics along four dimensions: exploration, revision, motion, and commitment, which together form a \emph{dynamical profile} of a given reasoning trajectory. 
We then track the dominance of the correct answer along each trajectory to distinguish different mechanisms of success and failure. We apply this framework to a diverse collection of open-weight language models and four discrete-answer benchmarks and find substantial heterogeneity in reasoning dynamics, even among traces with the same endpoint and similar entropy profiles. We further show that different objectives favor different dynamical profiles, and that training and inference choices systematically reshape these profiles.

Our contributions focus on the following three questions:
\begin{itemize}
    \item \textbf{What do conventional summaries of reasoning miss?} We show both theoretically and empirically that similar endpoint accuracy and entropy profiles can arise from different answer-distribution trajectories.
    \item \textbf{Do models exhibit systematic differences in reasoning dynamics?} We construct dynamical profiles of reasoning trajectories and use them to reveal systematic differences in how final answers are reached within and across model families and tasks.
    \item \textbf{How do training and inference choices reshape these dynamics?} We analyze how changes in training checkpoint, scale, and decoding temperature affect dynamical profiles. 
\end{itemize}

\section{Related work}

\paragraph{Studying reasoning beyond the final answer.}
Final-answer accuracy does not reveal how a model arrived at its prediction, motivating work that studies the reasoning process \cite{prasad2023receval,xia2025evaluating}. A natural observable signal is Chain-of-Thought (CoT), which exposes a sequence of intermediate natural-language steps before the final answer \cite{kojima2022large, wei2022chain}. CoT's textual content can be monitored for undesirable behavior \cite{baker2025monitoring, emmons2025chainthoughtnecessarylanguage, korbak2025chain}, trace length provides a signal of reasoning confidence \cite{devic2025tracelengthsimpleuncertainty}, agreement across sampled chains measures solution robustness \cite{wang2023selfconsistency}, and step-level verifier scores provide signals of local reasoning validity \cite{uesato2022solving, lightman2024let}. However, CoT may not reflect the actual reasons that determined the final answer \cite{turpin2023language, arcuschinchain, lanham2023measuring}. Complementary work instead studies the model’s internal computations, probing hidden representations for information about truthfulness, hallucination, future outputs, and latent reasoning concepts \cite{azaria2023the, marks2024geometrytruthemergentlinear, chen2024insidellmsinternalstates, pal2023future, gurnee2026verbalizable}, or intervening on these representations to steer model behavior \cite{zou2025representationengineeringtopdownapproach, turner2024steeringlanguagemodelsactivation}. In our work, we propose to study the reasoning process by tracking the evolution of the model's predictive distribution over candidate final answers as reasoning unfolds.

\paragraph{Information-theoretic views of reasoning.}
Recent work studies information-theoretic signals throughout the reasoning process. Information gain, mutual information, and token entropy have been used to identify important points of the reasoning process \cite{ton2025understandingchainofthoughtllmsinformation,qian2025demystifyingreasoningdynamicsmutual,NEURIPS2025_a797c2d2}. Entropy-based signals have been used to adapt exploration depth \cite{zhang2025entropy}, compress redundant reasoning steps \cite{ICLR2026_40b2844d}, and inform early stopping \cite{sharma2025thinkjustenoughsequencelevel,hosseini2026early}. Other work studies the temporal shape of uncertainty: intermediate-answer confidence exhibits different dynamics in correct and incorrect rollouts \cite{hosseini2026early}, early entropy trajectories characterize different reasoning regimes \cite{xia2026llms}, and theoretical work relates conditional answer-entropy dynamics to the accumulation of answer-relevant information \cite{català2026stepwiseinformativenessassumptionentropy}. However, entropy is limited because it does not capture which hypotheses carry the probability mass. To address this limitation, we instead study the full predictive distribution over candidate answers throughout reasoning.

\section{Answer-distribution trajectories: a stochastic-dynamics view of reasoning}
\label{sec:framework}

In this section, we formalize \emph{answer-distribution trajectories} and establish how they relate to endpoint prediction and entropy profiles.

\subsection{Answer-distribution trajectories}
\label{sec:trajectories}

We first define a reasoning model's predictive distribution over reasoning traces and final answers.

\begin{definition}[Model predictive distribution]
Given a query \(Q\), a reasoning model with parameters \(\theta\) generates a sequence of intermediate reasoning tokens
\(
C_{1:K} = (C_1,\dots,C_K)
\)
followed by an answer sequence
\(
A_{1:T} = (A_1,\dots,A_T),
\)
where $K$ and $T$ are stochastic sequence lengths. The corresponding autoregressive distribution over reasoning traces is given by 
$p_\theta(C_{1:K}\mid Q)
=
\prod_{k=1}^{K}
p_\theta(C_k\mid Q,C_{1:k-1}).$
Conditioned on a complete reasoning trace \(C_{1:K}\), the answer distribution factorizes as
$
p_\theta(A_{1:T}\mid Q,C_{1:K})
=
\prod_{t=1}^{T}
p_\theta(A_t\mid Q,C_{1:K},A_{1:t-1}).
$
\label{def:modelpredictive}
\end{definition}

In our empirical analysis, we apply a deterministic parser that maps each generated answer sequence \(A_{1:T}\) to a single discrete answer label \(Y \in \mathcal{A}\), where \(\mathcal{A}\) denotes the discrete answer space. We use this parsed answer label in the definitions that follow.

\begin{definition}[Prefix-conditioned predictive distribution]
For a prefix  \(C_{1:k}\), with \(k>0\), the \emph{prefix-conditioned predictive distribution} over discrete answer labels is
\(
p_k(a)
=
p_\theta(Y=a\mid Q,C_{1:k}),
\)
\(
a\in\mathcal{A}.
\)
This distribution marginalizes over all possible future reasoning continuations:
\(
p_k(a)
=
\sum_{c_{>k}}
p_\theta(C_{>k}=c_{>k}\mid Q,C_{1:k})
\,
p_\theta(
Y=a
\mid
Q,C_{1:k},C_{>k}=c_{>k}
),
\)
where \(C_{>k}\) denotes a variable-length future reasoning continuation.
\label{def:prefixpredictive}
\end{definition}

As the reasoning trace unfolds, these prefix-conditioned distributions form a sequence of distributions over the final answer. We refer to this sequence as the model’s \emph{answer-distribution trajectory}.

\begin{definition}[Answer-distribution trajectory]
For a question \(q\) and realized reasoning trace \(c_{1:K}\), we define the \emph{answer-distribution trajectory} as
$
\mathcal{T}(q,c)
=
\left(p_0,p_1,\ldots,p_K\right)
$,
where \(p_k\) denotes the prefix-conditioned predictive distribution associated with the reasoning prefix \(c_{1:k}\) and the query \(q\), as defined in Definition~\ref{def:prefixpredictive}.
\end{definition}

\subsection{Answer-distribution trajectories provide a finer representation}
\label{sec:information}

The answer-distribution trajectory retains substantially more structure than commonly used summaries of reasoning such as endpoint predictions or scalar uncertainty. We formalize this relationship below. Proofs of the theoretical results in this subsection are provided in Appendix~\ref{app:proofs}.

For an answer-distribution trajectory \(\mathcal{T}=(p_0,p_1,\ldots,p_K)\), we define its endpoint prediction as
$
E(\mathcal{T})
=
\argmax_{a\in\mathcal{A}} p_K(a),
$
and its entropy trajectory as
$
\mathcal{H}(\mathcal{T})
=
\bigl(H(p_0),H(p_1),\ldots,H(p_K)\bigr),
$
where
$
H(p)
=
-\sum_{a\in\mathcal{A}} p(a)\log p(a).
$

\begin{proposition}[Trajectory determines endpoint and entropy]
\label{prop:projection}
Both \(E\) and \(\mathcal{H}\) are deterministic functions of the answer-distribution trajectory \(\mathcal{T}\).
\end{proposition}

The converse, however, does not hold.

\begin{proposition}[Endpoint and entropy trajectory are non-identifying]
\label{prop:nonidentifiability}
Neither endpoint prediction nor entropy trajectory identifies the underlying answer-distribution dynamics. In particular:

\begin{enumerate}
    \item The endpoint map \(E\) is non-injective: there exist distinct answer-distribution trajectories \(\mathcal{T}\neq\mathcal{T}'\) such that
    $
    E(\mathcal{T})=E(\mathcal{T}').
    $
    Hence, endpoint prediction does not determine the sequence of predictive states that produced it.
    \item For \(|\mathcal{A}|\geq 2\), the entropy map \(\mathcal{H}\) is non-injective: there exist distinct answer-distribution trajectories \(\mathcal{T} \neq \mathcal{T}'\) such that \(\mathcal{H}(\mathcal{T}) = \mathcal{H}(\mathcal{T}')\). Hence, an entropy trajectory does not identify the underlying sequence of predictive states.
\end{enumerate}
\end{proposition}

\begin{corollary}[Answer-distribution trajectories are a strictly finer representation]
\label{cor:finer}
The answer-distribution trajectory is a strictly finer representation of the reasoning process than either endpoint prediction or entropy trajectory. Both are deterministic functions of the trajectory by Proposition~\ref{prop:projection}, while Proposition~\ref{prop:nonidentifiability} shows that neither is sufficient, in general, to recover the underlying trajectory.
\end{corollary}

Endpoint evaluation preserves the \emph{destination} of reasoning but not its path; entropy tracks \emph{concentration} but not which hypotheses carry the probability mass. Answer-distribution trajectories retain both: they capture which hypotheses remain plausible and how probability mass moves between them.

\section{Trajectory-level diagnostics}
\label{sec:diagnostics}

Answer-distribution trajectories allow us to define trajectory-level metrics that summarize different aspects of how predictive answer distributions evolve during reasoning. Taken together, these metrics form a \emph{dynamical profile} of a reasoning trajectory, providing a common representation for comparing reasoning behavior across models and tasks. By additionally orienting the trajectory relative to the correct answer, we can distinguish different dynamical mechanisms of reasoning success and failure.

\subsection{Dynamical reasoning profiles}
\label{sec:target_agnostic}

We summarize each trajectory along four complementary dimensions: \textbf{exploration}, \textbf{revision}, \textbf{motion}, and \textbf{commitment}. These values constitute the trajectory's \emph{dynamical profile}, providing a common coordinate system for characterizing reasoning dynamics.

Different dynamical patterns may arise across reasoning trajectories, and which profiles are preferable will depend on what we are optimizing for. Dynamical profiles make this heterogeneity explicit and measurable, allowing us to compare reasoning strategies and study which profiles are better suited to different objectives. The first property we consider is the breadth of a trajectory's predictive state.

\paragraph{Exploration: How many hypotheses remain in contention?} We measure this by the effective support of the predictive distribution,
$
S_k^{\mathrm{eff}} = \exp\left( -\sum_a p_k(a)\log p_k(a) \right) = \exp(H(p_k)),
$
which is the base-$e$ perplexity of the predictive answer distribution and can be interpreted as the effective number of plausible answers after $k$ CoT tokens. $S_k^{\mathrm{eff}}\approx 1$ indicates concentration on a single hypothesis, while larger values indicate broader competition. To summarize exploration over the full trajectory, we use the mean effective support across the CoT tokens,
\(
S_{\mathrm{eff}}
=
\frac{1}{m}
\sum_{j=1}^{m}
S^{\mathrm{eff}}_{k_j}.
\)

Breadth, however, does not tell us whether reasoning actually changes which hypothesis the model prefers. We thus consider a second property of reasoning traces.

\paragraph{Revision: Do the leading hypotheses change, and are earlier hypotheses revisited?} 
Let \(D_k=\arg\max_{a\in\mathcal A}p_k(a)\) denote the set of dominant answers after the first \(k\) CoT tokens. We first measure how often dominance changes,
\(
N_{\mathrm{switch}}
=
\sum_{k=0}^{K-1}
\mathbf{1}[D_{k+1} \neq D_k].
\)
To capture whether these switches return to previously considered hypotheses, we also measure recurrence,
\(
N_{\mathrm{return}}
=
\sum_{a\in\mathcal{A}}
\#\left\{
k \in \{1,\ldots,K\} :
a\in D_k,\;
a\notin D_{k-1},\;
\exists\, j<k \text{ such that } a\in D_j
\right\}.
\)

Revision records which answers are on top, but discards how probability mass moves underneath. Thus, we next consider the evolution of the full predictive distribution.

\paragraph{Motion: How much does reasoning change the predictive state?}
We measure consecutive distributional change using total variation, 
$
J_k = \mathrm{TV}(p_k,p_{k+1}) = \frac{1}{2} \sum_a \left| p_{k+1}(a)-p_k(a) \right|.
$
The total path length,
$
V = \sum_k J_k,
$
measures cumulative distributional movement. Its temporal concentration,
$
\mathrm{PR} = { \left(\sum_k J_k\right)^2 }/{ (K\sum_k J_k^2)},
$
distinguishes trajectories that evolve gradually from those whose movement is concentrated in a small number of steps.  To quantify movement directness, we compare start-to-end displacement with total path length, \(R_{\mathrm{direct}} = \frac{TV(p_0, p_K)}{V}.\) For \(V>0\), \(R_{\mathrm{direct}}\in[0,1]\), with higher values indicating more direct trajectories and lower values indicating greater backtracking.

Finally, motion does not tell us when the predictive state settles. Our final dimension captures this.

\paragraph{Commitment: When does competition resolve?} We measure when a hypothesis becomes both confident and stable. For a threshold $\tau \in (0.5,1]$, we define
\(
k_{\mathrm{commit}}(\tau)
=
\min\left\{
k :
\exists\, a\in\mathcal{A}
\text{ such that }
p_j(a)\geq\tau
\;\; \forall j\geq k
\right\}\), $t_{\mathrm{commit}} = \frac{k_{\mathrm{commit}}}{K}.
$
Small $t_{\mathrm{commit}}$ indicates early lock-in, whereas large values indicate persistent competition. Commitment is undefined if no hypothesis remains above the threshold through the end of the trajectory. 

These four dimensions describe how broadly probability is distributed across alternatives, how often the dominant answer changes, how much the predictive distribution shifts, and when it stabilizes. By turning these qualitative reasoning behaviors into measurable properties, dynamical profiles provide candidate targets for shaping reasoning behavior through training or inference-time interventions.

\subsection{Mechanisms of success and failure} 
\label{sec:success-failure-mechanisms} 

We now combine the evolution of the gold answer with the trajectory's outcome to define various dynamical mechanisms of success and failure. Let \(a^\star\) denote the gold answer and \(D_k\) the set of dominant answers at reasoning position \(k\). We say that the gold answer is \emph{dominant} at \(k\) whenever \(a^\star \in D_k\), including when it is tied with another answer. We quantify the prevalence of these mechanisms across models in Section~\ref{sec:coarse-summaries-results}.

\paragraph{Success mechanisms.} Conditioning on realized endpoint correctness, correct traces separate into four dynamical mechanisms:
\begin{itemize}
    \item \textbf{Stable success.} The gold answer remains dominant throughout the observed trajectory.
    \item \textbf{Rescue.} The gold answer is not initially dominant but is dominant at the end of the trajectory.
    \item \textbf{Detour success.} The gold answer is dominant initially and at the end of the trajectory, but temporarily loses dominance.
    \item \textbf{Fragile success.} The realized answer is correct even though the gold answer is not dominant at the end of the trajectory.
\end{itemize}

\paragraph{Failure mechanisms.}
Incorrect traces separate into four complementary failure mechanisms:
\begin{itemize}
    \item \textbf{Never discovered.} The gold answer is never dominant throughout the observed trajectory.
    \item \textbf{Escape.} The gold answer is initially dominant but is not dominant at the end of the trajectory.
    \item \textbf{Failed rescue.} The gold answer is not initially dominant, becomes dominant at some point during the trajectory, but is not dominant at the end.
    \item \textbf{Stochastic miss.} The realized answer is incorrect even though the gold answer is dominant at the end of the trajectory.
\end{itemize}

Several of these mechanisms are informative about \textit{overthinking} \cite{pmlr-v267-chen25bx, zhou-etal-2026-thinking, caldarella2026thinkingpastanswerevaluating}: escape and failed rescue capture cases in which continued reasoning moves the model away from a state where the correct answer was dominant, while detour success captures cases in which this degradation is temporary.

An important observation is that the same dynamical pattern can produce different outcomes depending on which hypotheses the trajectory explores and ultimately favors. Nevertheless, some dynamical strategies may increase the likelihood of success for a given model, task, or objective. Characterizing both profiles and mechanisms therefore provides a vocabulary for identifying advantageous reasoning behaviors and, ultimately, for encouraging them through training or inference-time interventions.

\section{Results}

In this section, we conduct a broad set of experiments to determine whether answer-distribution trajectories capture information missed by conventional summaries of reasoning, and what additional insights they provide. We organize our empirical validation section around three questions: (i) can traces with the same endpoint or similar entropy profiles exhibit different answer-distribution dynamics, (ii) how do reasoning dynamics vary across traces, models, and tasks, and (iii) how do instruction tuning, scale, and decoding temperature reshape these dynamics?

We evaluate sixteen models across four datasets (GSM8K, ARC, SVAMP, and MATH) spanning base, instruction-tuned, and RL-trained regimes. We estimate answer-distribution trajectories via Monte Carlo rollouts under stochastic decoding. Full details on the evaluation setup are provided in Appendix~\ref{sec:experimental_setup}, on figures and table reproduction in Appendix~\ref{app:visualization_reproduction}, and on computational cost in Appendix~\ref{sec:computational_cost}.

\subsection{Coarse summaries collapse distinct reasoning dynamics}
\label{sec:coarse-summaries-results}

Our theoretical analysis in Section~\ref{sec:information} showed that answer-distribution trajectories determine both endpoint predictions and entropy profiles, whereas the converse does not hold. We now test whether this non-identifiability arises in practice. 

\paragraph{Same endpoint, different dynamics.}
Figure~\ref{fig:mechanisms-by-model} compares the prevalence of the success and failure mechanisms defined in Section~\ref{sec:success-failure-mechanisms} across models, with models ordered by endpoint accuracy. Models with similar accuracy can nevertheless exhibit markedly different mixtures of success and failure mechanisms, as seen among neighboring rows in the accuracy-ordered panels. Endpoint performance therefore does not identify the dynamical processes through which successes and failures arise.

\begin{figure}[t]
    \centering
    \begin{subfigure}[t]{0.49\linewidth}
        \centering
        \includegraphics[width=\linewidth]{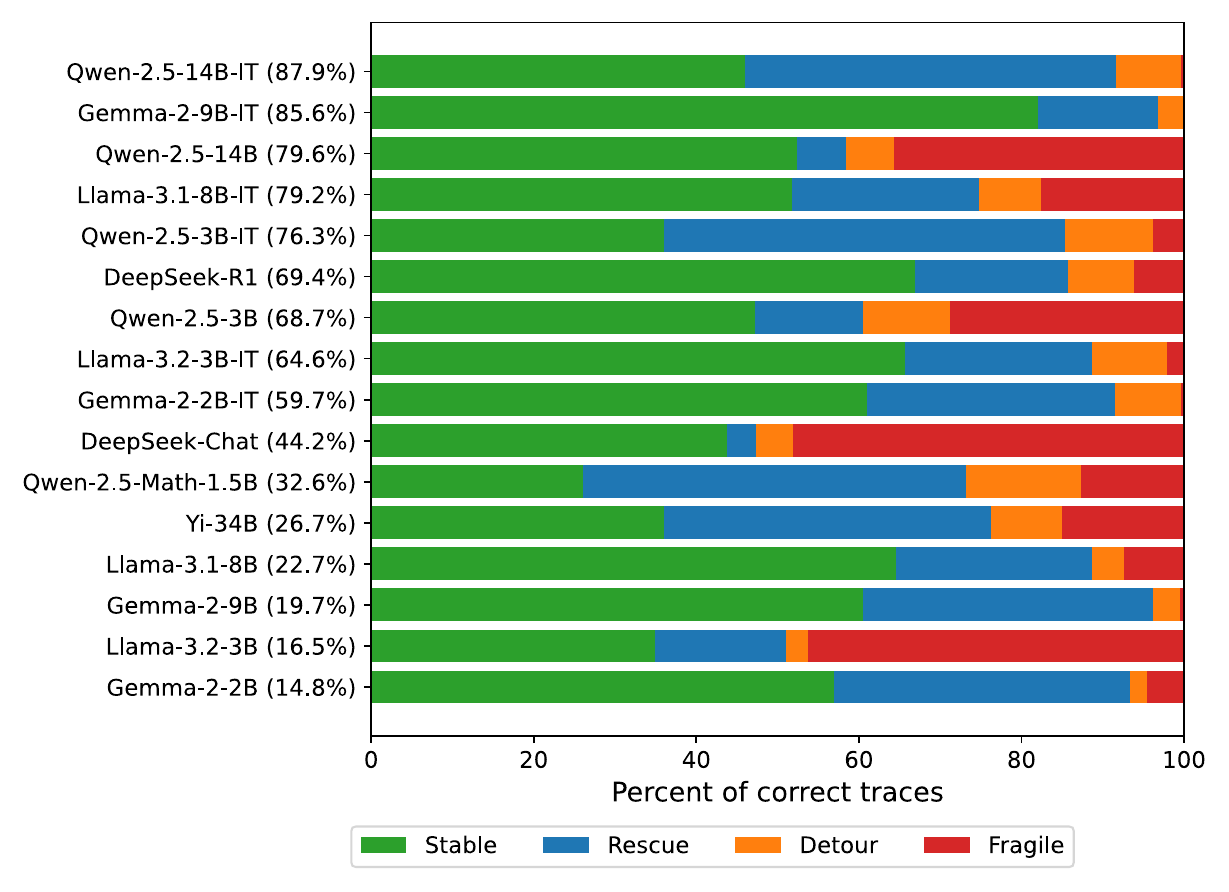}
        \caption{Success mechanisms among correct traces.}
        \label{fig:success-mechanisms-by-model}
    \end{subfigure}
    \hfill
    \begin{subfigure}[t]{0.49\linewidth}
        \centering
        \includegraphics[width=\linewidth]{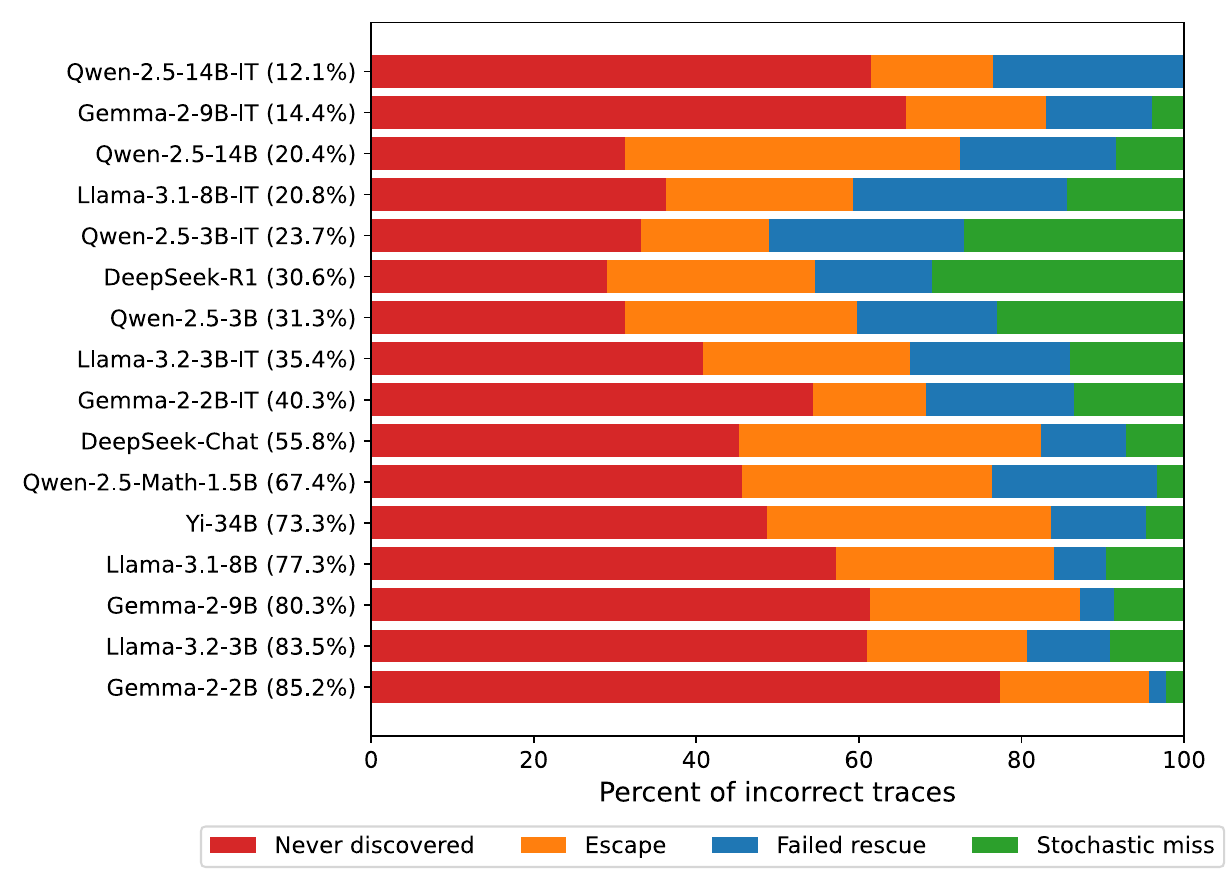}
        \caption{Failure mechanisms among incorrect traces.}
        \label{fig:failure-mechanisms-by-model}
    \end{subfigure}
    \caption{
    Success and failure mechanisms by model, with rows sorted by endpoint accuracy. Models with similar endpoint accuracy can exhibit substantially different mixtures of these mechanisms.
    }
    \label{fig:mechanisms-by-model}
\end{figure}

\begin{figure}[t]
    \centering
    \includegraphics[width=0.7\linewidth]{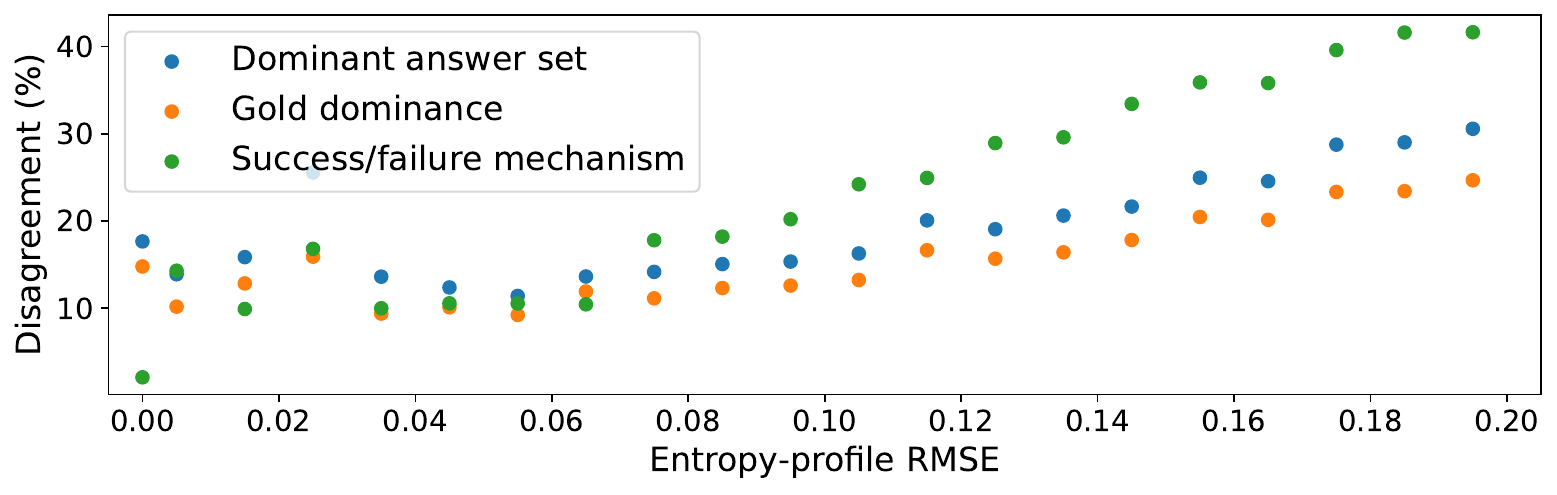}
    \caption{
    We compare cross-model traces for the same question using the RMSE between their entropy profiles over normalized reasoning time. The displayed range up to RMSE = 0.20 corresponds to 12\% of the largest same-question cross-model RMSE observed in our data. Similar entropy evolution does not imply similar answer-distribution dynamics.
    }
    \label{fig:entropy-ambiguity-curve}
\end{figure}

\paragraph{Same entropy profile, different dynamics.}
Figure~\ref{fig:entropy-ambiguity-curve} matches trajectories for the same question according to the similarity of their entropy profiles and measures disagreement in properties of the corresponding answer-distribution trajectories. Even for traces with similar entropy profiles, trajectories can disagree about which answers are dominant, whether the gold answer is dominant, and which success or failure mechanism the trajectory instantiates. Entropy therefore does not identify the underlying answer-distribution dynamics.

These results confirm that answer-distribution trajectories retain information lost by conventional summaries, motivating their use as a richer diagnostic of reasoning dynamics. In the following sections, we use them to characterize variation across traces, models, and tasks, and to study how training and inference choices reshape that variation.

\subsection{Reasoning dynamics vary within and across models and tasks} \label{sec:dynamical-landscape} 

The dynamical profiles introduced in Section~\ref{sec:target_agnostic} provide a coordinate system for comparing reasoning trajectories. We now ask whether traces exhibit distinct profiles within and across models and tasks, and whether this variation indicates which reasoning dynamics are better suited to different objectives.

\paragraph{Individual traces exhibit substantial variation.} 
Reasoning trajectories can differ substantially in their dynamical profiles. To identify the sources of this variation, in Table~\ref{tab:variation_components} we decompose each trajectory-level metric into four components: variation across model-level means, variation across task-level means, model-task interaction, and variation across individual traces within a fixed model-task condition. Every reported variance component is nonzero, indicating heterogeneity in each metric. Notably, the within-condition component is the largest for every metric, which indicates that a model-task pair does not correspond to a single characteristic reasoning profile. This suggests that the metrics capture meaningful variation at the level of individual reasoning traces and are not limited to distinguishing between models or tasks.

\begin{table}[t]
\centering
\caption{Magnitude and sources of variation in trajectory-level reasoning metrics. Entries report root variance components in the natural units of each metric. The total variation is obtained by summing the squared components and taking the square root. Model-task pairs are weighted equally.}
\label{tab:variation_components}
\small
\resizebox{\columnwidth}{!}{ 
\begin{tabular}{lrrrrrrr}
\toprule
Source of variation & $S_{\mathrm{eff}}$ & $N_{\mathrm{switch}}/\mathrm{step}$ & $N_{\mathrm{return}}/\mathrm{step}$ & $V/\mathrm{step}$ & $\mathrm{PR}$ & $R_{\mathrm{direct}}$ & $t_{\mathrm{commit}}(0.8)$ \\
\midrule
Model & 0.529 & 0.160 & 0.021 & 0.164 & 0.211 & 0.159 & 0.160 \\
Task & 0.323 & 0.079 & 0.017 & 0.068 & 0.070 & 0.087 & 0.065 \\
Model x task & 0.466 & 0.129 & 0.034 & 0.103 & 0.092 & 0.104 & 0.167 \\
Within model-task & 1.09 & 0.270 & 0.123 & 0.204 & 0.254 & 0.299 & 0.322 \\
\midrule
Total & 1.34 & 0.348 & 0.131 & 0.290 & 0.350 & 0.365 & 0.402 \\
\bottomrule
\end{tabular}
}
\end{table}

\paragraph{Which dynamical profile is favorable depends on the objective being optimized.} As shown in Table~\ref{tab:objective_profiles}, model-task pairs with the highest endpoint accuracy tend to exhibit more focused reasoning dynamics: they maintain narrower predictive support, switch dominant hypotheses less frequently, undergo less local distributional movement, and commit earlier than the task average. In contrast, when the goal is efficiency, the shortest-CoT traces exhibit broader predictive support, more frequent hypothesis switching, greater and more temporally distributed local distributional movement, and later commitment. Thus, no single reasoning profile is universally favorable.

These examples illustrate why dynamical profiles can be useful beyond describing heterogeneity. By turning qualitative reasoning behaviors into measurable properties, they provide coordinates to evaluate which dynamical profiles are advantageous under a particular model, task, and objective.

\begin{table}[t]
\centering
\caption{Different objectives favor different dynamical profiles. Standardized mean profiles for high-accuracy and short-CoT model-task pairs. Values denote deviations from the task mean in standard deviation units.}
\label{tab:objective_profiles}
\small
\resizebox{\columnwidth}{!}{ 
\begin{tabular}{lrrrrrrrrr}
\toprule
Objective & $S_{\mathrm{eff}}$ & $N_{\mathrm{switch}}/\mathrm{step}$ & $N_{\mathrm{return}}/\mathrm{step}$ & $V/\mathrm{step}$ & $\mathrm{PR}$ & $R_{\mathrm{direct}}$ & $t_{\mathrm{commit}}(0.8)$ & Acc. & CoT length \\
\midrule
High accuracy & -0.83 & -0.68 & -0.50 & -0.85 & -0.92 & -0.14 & -0.53 & 78.8\% & 240 \\
Short CoT & +1.09 & +1.03 & -0.45 & +1.23 & +1.27 & +1.30 & +0.76 & 18.6\% & 56 \\
\bottomrule
\end{tabular}
}
\end{table}

\subsection{Training and inference choices reshape reasoning dynamics}
\label{sec:reshaping-dynamics}

The previous section shows that traces can present vastly different reasoning profiles and that the desirability of each of them depends on what we seek to optimize. We next ask whether choices made during training or inference time can steer models towards presenting specific dynamical profiles. 

Figure~\ref{fig:training_inference_dynamics} reports the change in each trajectory-level metric for every matched comparison, together with the aggregate effect. Instruction tuning produces the clearest dynamical shift, yielding narrower predictive support (\(S_{\mathrm{eff}}\)), fewer hypothesis switches (\(N_{\mathrm{switch}}/\mathrm{step}\)), less distributional movement (\(V/\mathrm{step}\)), earlier commitment (\(t_{\mathrm{commit}}(0.8)\)), and higher endpoint accuracy. Increasing model scale shows a weaker, less uniform tendency toward shifts in the same direction, with improved accuracy. In contrast, increasing decoding temperature broadens predictive support (\(S_{\mathrm{eff}}\)), increases distributional movement (\(V/\mathrm{step}\)), and delays commitment (\(t_{\mathrm{commit}}(0.8)\)), while its effect on hypothesis switching (\(N_{\mathrm{switch}}/\mathrm{step}\)) is smaller and accuracy changes are model-dependent. These results suggest that training and inference-time interventions induce systematic average shifts in the dynamical profiles of the generated reasoning traces.

\begin{figure*}[t]
    \centering
    \includegraphics[
        width=\textwidth,
        trim=0 0 0 0,
        clip
    ]{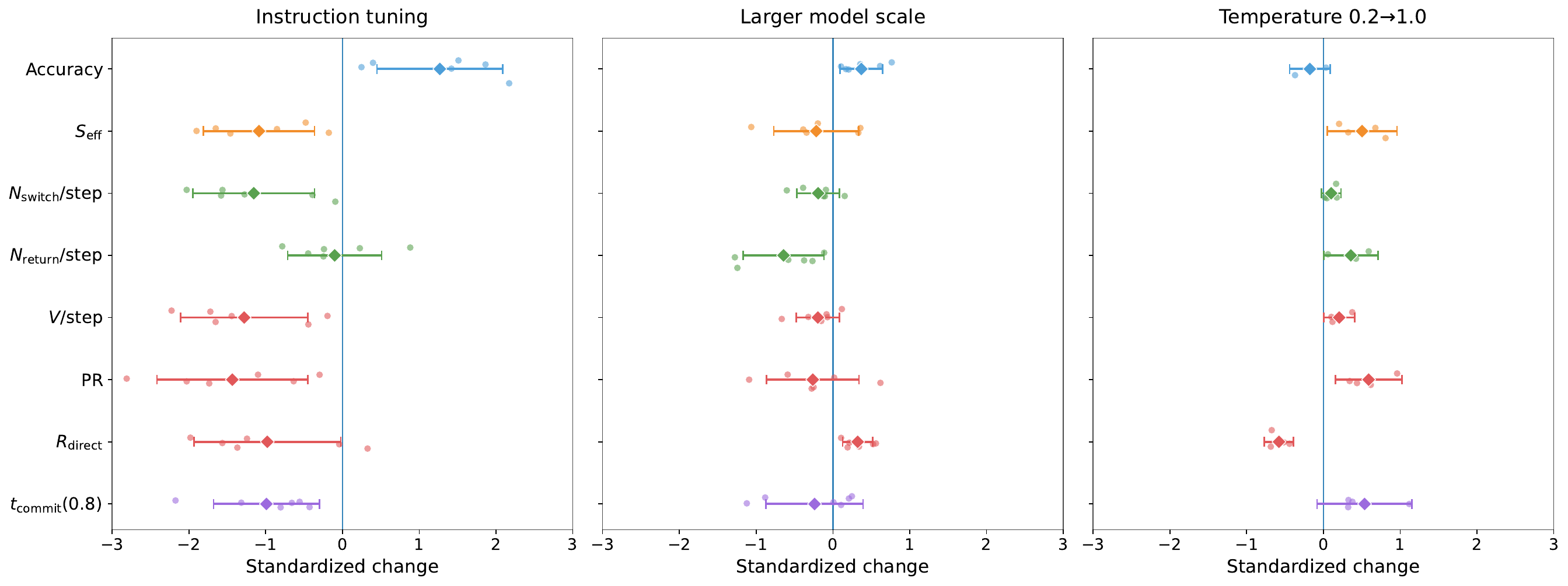}
    \caption{
        \textbf{Training and inference choices reshape reasoning dynamics.}
        Standardized changes in endpoint accuracy and trajectory-level metrics
        under matched comparisons of instruction tuning, increasing model scale,
        and increasing decoding temperature.
        Each point represents a matched model comparison; diamonds show mean effects and error bars denote 95\% confidence intervals. Changes are standardized by the across-condition standard deviation of the corresponding metric.
    }
    \label{fig:training_inference_dynamics}
\end{figure*}

\section{Conclusion and Open Questions}

This work introduces answer-distribution trajectories as a framework for studying how a language model’s predictive distribution over final answers evolves throughout CoT reasoning. Unlike endpoint predictions or entropy profiles, these trajectories preserve both which hypotheses are supported and how probability mass moves among them. We derive trajectory-level measures of exploration, revision, motion, and commitment and combine them into dynamical profiles. We use these profiles to characterize reasoning dynamics across sixteen open-weight language models and four reasoning benchmarks, finding variation in how final answers are reached both within and across model-task pairs. We also show that different objectives favor different dynamical profiles, and that training and inference choices systematically reshape them. These results position the evolution of the full predictive distribution as an informative characterization of LLM reasoning, and stochastic dynamics as a natural mathematical language for describing the paths by which LLMs arrive at their answers.

Some open questions remain. Our analysis is computationally expensive, developing cheaper approximations of answer-distribution trajectories is therefore an important direction for future work. Moreover, the comparisons in Section~\ref{sec:reshaping-dynamics} are only descriptive, and future controlled studies could better isolate the effects of architecture, scale, and post-training. Lastly, whether the information that answer-distribution trajectories provide can translate into practical gains remains an open question. Promising directions include improving error prediction and adaptive test-time compute, diagnosing and mitigating overthinking, and informing reasoning-model optimization and post-training.

\section*{Acknowledgements}

Mar Gonzàlez I Català acknowledges that this project was supported by G-Research and by the Qualcomm Innovation Fellowship Europe. Davide Murari acknowledges support from the EPSRC grant EP/Y028783/1. George D. Monta\~nez acknowledges support from the William Whewell Centre for Science and Natural Theology and the Global Scholars Foundation.

%% file: AppendixA_experimental_setup.tex
\section{Experimental Setup}
\label{sec:experimental_setup}

\subsection{Tasks and datasets}
\label{sec:datasets}

We focus on reasoning tasks with a discrete answer space \(\mathcal{A}\). Each example consists of a question \(Q \in \mathcal{Q}\) and a ground-truth answer \(A \in \mathcal{A}\). We evaluate on the following datasets:

\begin{itemize}
\item \textbf{GSM8K} \cite{cobbe2021trainingverifierssolvemath}: grade-school mathematical word problems with numeric answers.
\item \textbf{ARC} \cite{clark2018thinksolvedquestionanswering}: multiple-choice science questions.
\item \textbf{SVAMP} \cite{patel2021nlpmodelsreallyable}: arithmetic word problems designed to test robustness to linguistic variation.
\item \textbf{MATH} \cite{hendrycks2021measuringmathematicalproblemsolving}: competition-level mathematics (we use algebra track).
\end{itemize}

For all datasets, we use the official test splits and apply deterministic answer normalization and parsing to map model outputs to discrete answer labels (e.g., numeric normalization for GSM8K, SVAMP and MATH, letter-to-option mapping for ARC). Invalid or unparsable outputs are mapped to a special null answer category.

\subsection{Models}
\label{sec:models}

We evaluate a diverse set of sixteen open-weight LLMs corresponding to different training regimes:

\begin{itemize}
\item \textbf{Gemma-2-2B} and \textbf{Gemma-2-9B} \cite{gemmateam2024gemma2improvingopen}: base and instruction-tuned variants for each.
\item \textbf{LLaMA-3.2-3B} and \textbf{LLaMA-3.1-8B} \cite{grattafiori2024llama3herdmodels}: base and instruction-tuned variants for each.
\item \textbf{Qwen-2.5-3B} and \textbf{Qwen-2.5-14B} \cite{qwen2025qwen25technicalreport}: base and instruction-tuned variants for each.
\item \textbf{Qwen-2.5-Math-1.5B} \cite{qwen2025qwen25technicalreport}: SFT-trained specialized on math problems.
\item \textbf{DeepSeek-Chat-7B} \cite{deepseekai2024deepseekllmscalingopensource}: SFT-trained chat model.
\item \textbf{DeepSeek-R1-distilled-7B} \cite{Guo_2025}: reasoning-specialized RL model.
\item \textbf{Yi-1.5-34B} \cite{ai2025yiopenfoundationmodels}: base variant.
\end{itemize}

Base models correspond to pretrained LLMs without supervised or reinforcement fine-tuning. Instruction-tuned (IT) models are supervised fine-tuned on instruction-following data. RL-trained models are optimized using reinforcement learning from human or synthetic feedback.

\subsection{Generation procedure}
\label{subsec:generation_procedure}

For each question \(Q=q\), we sample one reasoning trajectory from the model under a fixed stochastic decoding configuration (temperature, nucleus sampling, and maximum generation length). Concretely, we draw
\[
C_{1:K} \sim p_\theta(\cdot \mid q),
\]
where \(K\) denotes the generated reasoning length (up to a fixed truncation limit). We treat this trajectory \(C_{1:K}\) as one realization of the model's reasoning process for the given query.

Unless otherwise specified, decoding uses:
\begin{itemize}
    \item temperature $T=0.7$
    \item nucleus sampling with $p=0.9$
    \item a maximum generation length of 600 tokens
\end{itemize}

All continuation rollouts used to estimate prefix-conditioned answer distributions use the same decoding configuration to ensure comparability.

\subsection{Monte-Carlo estimation of prefix-conditioned answer distributions}
\label{subsec:mc_answer_distribution}

Given a fixed query \(Q=q\) and a realized reasoning prefix \(C_{1:k}=c_{1:k}\), the model induces a predictive distribution over discrete final-answer labels,
\[
p_k(a) 
=
p_\theta(Y=a \mid q,c_{1:k}),
\qquad a\in\mathcal{A},
\]
where \(Y\) denotes the discrete answer label obtained by applying the deterministic answer parser to a generated continuation.

Because this distribution marginalizes over all possible future reasoning continuations, we approximate it using Monte-Carlo sampling. For each fixed prefix \((q,c_{1:k})\), we draw \(N\) independent stochastic continuations from the model,
\[
C^{(i)}_{>k}, A^{(i)}
\sim
p_\theta(\cdot \mid q,c_{1:k}),
\qquad i=1,\dots,N,
\]
using the same decoding configuration as the original reasoning trajectory. Each generated answer sequence \(A^{(i)}\) is mapped by the deterministic parser to a discrete answer label \(Y^{(i)}\in\mathcal{A}\), including the special null category for invalid or unparsable outputs.

The samples define the empirical prefix-conditioned answer distribution
\[
\widehat{p}_k(a)
=
\frac{1}{N}
\sum_{i=1}^{N}
\mathbf{1}\!\left\{Y^{(i)}=a\right\},
\qquad a\in\mathcal{A}.
\]

We use \(\widehat{p}_k\) as the empirical approximation to the prefix-conditioned predictive state \(p_k\) throughout our trajectory analysis. All trajectory-level quantities are computed from these empirical predictive distributions.

All continuation rollouts are performed in evaluation mode without gradient computation, and sampling parameters are held fixed across models and prefixes. In practice, we use \(N=16\) independent continuations per prefix unless otherwise stated.

\subsection{Checkpointed prefix evaluation}
\label{subsec:checkpointed_prefix_evaluation}

Estimating conditional answer entropy at every token is computationally expensive. We therefore evaluate at checkpoint positions
\[
\mathcal{J} = \{j_1, j_2, \dots, j_m\} \subseteq \{0,1,\dots,K\},
\]
spaced uniformly at stride \(s=16\), and always including the final prefix length of the trajectory (\(j_m = K\)). The checkpoint at position \(0\) corresponds to the empty prefix.

%% file: AppendixB_reproduction_figures_and_tables.tex
\section{Reproduction details for figures and tables}
\label{app:visualization_reproduction}

This section documents the exact construction of the data-driven figures and tables in the main text.

\paragraph{Common trace processing.}
The analyses use the complete $4\times16$ design consisting of ARC, GSM8K, MATH, and SVAMP crossed with the sixteen paper models. Answer parsing, normalization, and correctness evaluation follow the procedure described in Section~\ref{sec:datasets}.

Trajectory-level metrics are defined in Section~\ref{sec:target_agnostic}. In all analyses below, these metrics are computed from the observed answer-distribution states at checkpoint positions with a nonempty predictive distribution. $R_{\mathrm{direct}}$ is undefined for trajectories with zero total path length and is omitted from summaries of that metric. Commitment uses the threshold $\tau=0.8$ and is undefined when no hypothesis remains above this threshold through the end of the observed trajectory; such traces are omitted from summaries of $t_{\mathrm{commit}}$. In order to compare across trajectories of different lengths, \(N_{\mathrm{switch}}\), \(N_{\mathrm{return}}\), and \(V\) are normalized by the number of observed checkpoint transitions.

\paragraph{Figure~\ref{fig:mechanisms-by-model}: success and failure mechanisms.}
Each trace is first split by realized endpoint correctness. Correct and incorrect traces are assigned to the success and failure mechanisms defined in Section~\ref{sec:success-failure-mechanisms} based on the trajectory of gold-answer dominance. At each checkpoint, the gold answer is considered dominant whenever it belongs to the set of probability-maximizing answers, including in the case of an exact tie. For each model, the left panel reports the fraction of correct traces in each success mechanism, and the right panel reports the corresponding fractions among incorrect traces. The rows are ordered by the model's pooled endpoint accuracy across the four default tasks. The percentage displayed next to each model name is pooled accuracy in the success panel and pooled failure rate in the failure panel.

\paragraph{Figure~\ref{fig:entropy-ambiguity-curve}: near-matched entropy profiles.}
We compare reasoning traces from different models only when they correspond to the same dataset example. Because traces can have different lengths and checkpoints at different positions, we rescale reasoning progress from 0 (start) to 1 (end) and interpolate each entropy profile onto the same 101 evenly spaced points. We then measure how similar two entropy profiles are using root mean squared error (RMSE), where smaller values indicate more similar profiles. At each normalized position, the dominant set and gold-dominance indicator are taken from the most recent observed checkpoint.

For every cross-model pair, we record three disagreement measures: the fraction of normalized-time grid points with different dominant sets, the fraction with different gold-dominance indicators, and an indicator for whether the two traces instantiate different success or failure mechanisms. Mechanism disagreement is evaluated only when the two traces have the same realized outcome and both receive a mechanism label. We group non-identical trace pairs into non-overlapping RMSE bins of width $0.01$, from $(0,0.01]$ through $(0.19,0.20]$, so that each pair contributes to exactly one bin. Exact entropy-profile matches, defined as RMSE $\leq 10^{-12}$, are reported separately at zero. Each plotted point shows the mean disagreement among the pairs in that RMSE bin, expressed as a percentage, and is positioned at the midpoint of the corresponding bin.

\paragraph{Table~\ref{tab:variation_components}: variance decomposition.}
Values are computed separately for each of the seven trajectory-level profile metrics. Let $\mu_{mt}$ denote the mean of a metric in model $m$ and task $t$, let $\mu_{m\cdot}$ and $\mu_{\cdot t}$ be the corresponding marginal means, and let $\mu$ be the grand mean across the 64 model-task cells, and let \(M\) and \(T\) denote the numbers of models and tasks. The between-condition components are
\[
\sigma^2_{\mathrm{model}}
= \frac{1}{M}\sum_m(\mu_{m\cdot}-\mu)^2,
\qquad
\sigma^2_{\mathrm{task}}
= \frac{1}{T}\sum_t(\mu_{\cdot t}-\mu)^2,
\]
\[
\sigma^2_{\mathrm{interaction}}
= \frac{1}{MT}\sum_{m,t}
(\mu_{mt}-\mu_{m\cdot}-\mu_{\cdot t}+\mu)^2.
\]
The within-condition component is the equally weighted average of the population variances within the 64 model-task cells,
\[
\sigma^2_{\mathrm{within}}
= \frac{1}{MT}\sum_{m,t}\mathrm{Var}(X\mid m,t).
\]
The total variance is the sum of these four components. The table reports the square root of each component so that every entry remains in the natural units of the corresponding metric. Each model-task cell receives equal weight regardless of the number of traces in that cell.

\paragraph{Table~\ref{tab:objective_profiles}: objective-conditioned dynamical profiles.}
The trace-level data are first aggregated to one row for each of the 64 default model-task conditions. Mean CoT length is the mean number of generated tokens. CoT length is log transformed before standardization. Endpoint accuracy, log CoT length, and each profile metric are then standardized separately within each task across the sixteen model conditions using population standard deviation.

The \emph{High accuracy} group is the 25\% of all model-task pairs with the largest within-task standardized accuracy. The \emph{Short CoT} group is the 25\% with the smallest within-task standardized log CoT length. The reported profile entries are the average standardized values of the seven trajectory metrics within each selected group, where each metric is expressed relative to the mean and standard deviation for that task. The final two columns report unstandardized mean endpoint accuracy and unstandardized mean CoT length for interpretability.

\paragraph{Figure~\ref{fig:training_inference_dynamics}: training and inference interventions.}
Figure~\ref{fig:training_inference_dynamics} operates on model-task condition means. For instruction tuning, we compare six matched base and instruction-tuned model pairs: Gemma-2-2B vs.\ Gemma-2-2B-IT, Gemma-2-9B vs.\ Gemma-2-9B-IT, LLaMA-3.2-3B vs.\ LLaMA-3.2-3B-IT, LLaMA-3.1-8B vs.\ LLaMA-3.1-8B-IT, Qwen-2.5-3B vs.\ Qwen-2.5-3B-IT, and Qwen-2.5-14B vs.\ Qwen-2.5-14B-IT. Each pair is evaluated on all four tasks. For model scale, we compare Gemma-2-2B vs.\ Gemma-2-9B, Gemma-2-2B-IT vs.\ Gemma-2-9B-IT, LLaMA-3.2-3B vs.\ LLaMA-3.1-8B, LLaMA-3.2-3B-IT vs.\ LLaMA-3.1-8B-IT, Qwen-2.5-3B vs.\ Qwen-2.5-14B, and Qwen-2.5-3B-IT vs.\ Qwen-2.5-14B-IT, again evaluated on all four tasks. For decoding temperature, we use GSM8K runs from DeepSeek-R1-distilled-7B, Gemma-2-2B-IT, Qwen-2.5-3B, and Qwen-2.5-3B-IT, comparing $T=0.2$ with $T=1.0$.

For every task-specific matched comparison and metric, the raw change is defined as the intervened condition minus the reference condition. To place metrics with different units on a common horizontal scale, each raw change is divided by the sample standard deviation of that metric across the 64 default-decoding model-task condition means. For instruction tuning and model scale, these standardized changes are then averaged across the four tasks within each matched model pair. Each plotted point therefore represents one matched model-pair comparison averaged across tasks. For temperature, each point represents one model comparison on GSM8K.

The diamond is the arithmetic mean across matched model comparisons. Error bars are two-sided 95\% confidence intervals computed as the mean plus or minus the Student-$t$ critical value times the standard error, with degrees of freedom $n-1$. Thus, $n=6$ for the instruction-tuning and model-scale panels and $n=4$ for the temperature panel. Vertical jitter is visual only and is controlled by a fixed random seed. The three panels share a symmetric x-axis range so that effect magnitudes are visually comparable across instruction tuning, scale, and temperature.

%% file: AppendixC_computational_cost.tex
\section{Computational Cost}
\label{sec:computational_cost}

Estimating an answer-distribution trajectory requires repeated continuation sampling from intermediate reasoning prefixes. In this section, we characterize the computational cost of our procedure. 

For a trace of length $K$, evaluated every $s$ tokens, let
\[
m = \left\lceil \frac{K}{s} \right\rceil + 1
\]
denote the number of evaluated prefixes, including the empty prefix. With $N$ Monte-Carlo continuations per prefix, estimating one trajectory requires approximately
\[
N_{\mathrm{rollout}} = Nm
=
N\left(\left\lceil \frac{K}{s}\right\rceil + 1\right)
\]
continuation rollouts, in addition to generating the original reasoning trace. Thus, the number of continuation generations scales linearly with both $N$ and
the number of evaluated checkpoints, and approximately as $O(NK/s)$ in the
trace length.

These counts characterize the number of continuation generations. Actual wall-clock cost additionally depends on model size, batching, hardware, and the lengths of the sampled continuations.

The token-generation cost also depends on where a checkpoint occurs. Early
prefixes generally require longer continuations than late prefixes. If $L_r$ denotes the number of newly generated tokens in a continuation from checkpoint $r$, then the continuation-generation cost is proportional to
\[
N\sum_{r=1}^{m} L_r.
\]
Under the simplifying assumption that expected continuation length is proportional to the remaining trace length, uniformly spaced checkpoints imply a generated-token cost of order
\[
O\!\left(\frac{N K^2}{s}\right).
\]

%% file: AppendixD_proofs.tex
\section{Proofs}
\label{app:proofs}

\begin{proof}[Proof of Proposition~\ref{prop:projection}]
The endpoint prediction \(E(\mathcal{T})\) depends only on the terminal state \(p_K\), while each component of \(\mathcal{H}(\mathcal{T})\) depends only on the corresponding state \(p_k\). Therefore both are uniquely determined by \(\mathcal{T}\).
\end{proof}

\begin{proof}[Proof of Proposition~\ref{prop:nonidentifiability}]
Both claims can be witnessed on the binary answer space \(\mathcal{A}=\{a,b\}\).

For the first claim, consider
\[
\mathcal{T}_1:
(0.9,0.1)
\rightarrow
(0.9,0.1)
\rightarrow
(0.9,0.1),
\]
and
\[
\mathcal{T}_2:
(0.1,0.9)
\rightarrow
(0.4,0.6)
\rightarrow
(0.9,0.1).
\]
Both terminate with \(a\) as the dominant answer, and therefore
\[
E(\mathcal{T}_1)=E(\mathcal{T}_2)=a.
\]
Nevertheless, \(\mathcal{T}_1\) maintains the same dominant hypothesis throughout, whereas \(\mathcal{T}_2\) revises from \(b\) to \(a\). Hence the endpoint does not identify the preceding distribution dynamics.

For the second claim, take
\[
p=(0.9,0.1),
\qquad
q=(0.1,0.9),
\]
and consider the constant trajectories
\[
\mathcal{T}_1 : p \to p \to p,
\qquad
\mathcal{T}_2 : q \to q \to q.
\]
Since entropy is invariant under permutation of coordinates,
\[
H(p)=H(q),
\]
and therefore
\[
\mathcal{H}(\mathcal{T}_1)
=
\mathcal{H}(\mathcal{T}_2).
\]
However, \(\mathcal{T}_1 \neq \mathcal{T}_2\), since \(p\neq q\). Hence entropy profiles do not identify which hypotheses carry the predictive mass.
\end{proof}

%% file: AppendixE_licences.tex
\section{Licenses}
\label{subsec:licenses}

We do not introduce or release any new datasets or model checkpoints. All experiments use publicly released benchmark datasets and open-weight model checkpoints under their respective licenses. We use these artifacts only for academic evaluation and do not redistribute the original dataset contents or model weights. Table~\ref{tab:licenses} summarizes the licenses and usage terms for all artifacts used in our experiments.

\begin{table}[h]
\centering
\small
\caption{Licenses and usage terms for datasets and open-weight model checkpoints used in our experiments.}
\label{tab:licenses}
\resizebox{\textwidth}{!}{%
\begin{tabular}{lll}
\toprule
Artifact & Source & License / terms \\
\midrule
GSM8K & OpenAI & \href{https://huggingface.co/datasets/choosealicense/licenses/blob/main/markdown/cc-by-sa-4.0.md}{MIT License} \\
ARC & AllenAI AI2 ARC & \href{https://huggingface.co/datasets/choosealicense/licenses/blob/main/markdown/cc-by-sa-4.0.md}{CC BY-SA 4.0} \\
SVAMP & Patel et al. / HF mirror & \href{https://huggingface.co/datasets/choosealicense/licenses/blob/main/markdown/cc-by-sa-4.0.md}{MIT License} \\
MATH & Hendrycks et al. & \href{https://huggingface.co/datasets/choosealicense/licenses/blob/main/markdown/cc-by-sa-4.0.md}{MIT License} \\
\midrule
Gemma-2-2B, Gemma-2-9B & Google & \href{https://ai.google.dev/gemma/terms}{Gemma Terms of Use} \\
LLaMA-3.2-3B & Meta & \href{https://www.llama.com/llama3_2/license/}{Llama 3.2 Community License} \\
LLaMA-3.1-8B & Meta & \href{https://www.llama.com/llama3_1/license/}{Llama 3.1 Community License} \\
Qwen-2.5-3B & Qwen & \href{https://huggingface.co/Qwen/Qwen2.5-3B/blob/main/LICENSE}{Qwen Research} \\
Qwen-2.5-14B & Qwen & \href{https://huggingface.co/Qwen/Qwen2.5-14B/blob/main/LICENSE}{Apache 2.0} \\
Qwen-2.5-Math-1.5B & Qwen & \href{https://huggingface.co/Qwen/Qwen2.5-Math-1.5B/blob/main/LICENSE}{Apache 2.0} \\
DeepSeek-Chat-7B & DeepSeek & \href{https://github.com/deepseek-ai/DeepSeek-LLM/blob/main/LICENSE-MODEL}{DeepSeek Model License}; \href{https://github.com/deepseek-ai/DeepSeek-LLM/blob/main/LICENSE-CODE}{code under MIT} \\
DeepSeek-R1-distilled-7B & DeepSeek & \href{https://huggingface.co/datasets/choosealicense/licenses/blob/main/markdown/mit.md}{MIT License} \\
Yi-1.5-34B & 01.AI & \href{https://huggingface.co/datasets/choosealicense/licenses/blob/main/markdown/apache-2.0.md}{Apache 2.0} \\
\bottomrule
\end{tabular}
}
\end{table}

%% file: AppendixF_impact_statement.tex
\section{Impact Statement}
\label{app:broader_impacts}
This paper aims to advance the field of Machine Learning. While our work has potential societal implications, we do not identify any specific concerns that require particular emphasis at this stage.